\documentclass{article}

\PassOptionsToPackage{numbers, compress}{natbib}

\usepackage[preprint]{neurips_2026}
\usepackage[utf8]{inputenc} 
\usepackage[T1]{fontenc}    
\usepackage{url}            
\usepackage{booktabs}       
\usepackage{amsfonts}       
\usepackage{nicefrac}       
\usepackage{microtype}      
\usepackage[dvipsnames]{xcolor}  

\usepackage{graphicx}
\usepackage{subcaption}
\usepackage{wrapfig}
\usepackage[ruled,vlined]{algorithm2e}

\usepackage{hyperref}
\usepackage{multicol}
\usepackage{multirow}
\usepackage{float}

\usepackage{amsmath}
\usepackage{amssymb}
\usepackage{mathtools}
\usepackage{amsthm}
\usepackage[capitalize,noabbrev]{cleveref}

\theoremstyle{plain}
\newtheorem{theorem}{Theorem}[section]

\theoremstyle{definition}

\theoremstyle{remark}

\title{Mamba with Hierarchical Memory: Solving Representation Bottleneck in Long Sequence Modeling}

\author{%
  Qinwen Wang\thanks{These authors contributed equally.} \\
  \small The Hong Kong University of\\
  \small Science and Technology, Guangzhou\\
  \small \texttt{qwang477@connect.hkust-gz.edu.cn}
  \And
  Jieping Luo\textsuperscript{*} \\
  \small University of Oxford\\
  \small \texttt{jieping.luo@reuben.ox.ac.uk}
  \And
  Aoxiang Qin\textsuperscript{*} \\
  \small The Chinese University\\
  \small of Hong Kong, Shenzhen\\
  \small \texttt{aoxiangqin@link.cuhk.edu.cn}
  \And
  Ruoyu Zhao \\
  \small City University \\
  \small of Hong Kong\\
  \small \texttt{ruoyuzhao8-c@my.cityu.edu.hk}
  \And
  Jianxiong Tang \\
  \small City University \\
  \small of Hong Kong\\
  \small \texttt{jiatang@cityu.edu.hk}
  \And
  Wei Zhang \\
  \small Hainan Bielefeld University\\
  \small of Applied Sciences\\
  \small \texttt{wei.zhang@hainan-biuh.edu.cn}
  \And
  Zhichao Lu \\
  \small City University \\
  \small of Hong Kong\\
  \small \texttt{zhichao.lu@cityu.edu.hk}
  \And
  Luziwei Leng\thanks{Corresponding author.} \\
  \small BrainGalaxy\\
  \small \texttt{lengluziwei@braingalaxy.ai}
}

\begin{document}

\maketitle

\begin{abstract}
Recurrent linear attention models (RLAs) such as Mamba offer efficient linear-time sequence modeling as an alternative to Transformers, yet their fixed‑capacity recurrent states limit long‑sequence modeling. Drawing inspiration from hierarchical human memory, we propose Hierarchical Memory Mamba (HMM) to address this limitation. Building upon a pre‑trained Mamba backbone, HMM integrates a lightweight working memory that extracts slow paragraph‑level semantics (PLS) from the fast sensory memory embedded in the backbone’s hidden states. The PLS is subsequently compressed into persistent long‑term memory for task‑relevant retrieval. The hierarchical processing of semantic information overcomes the representation bottleneck of RLAs and endows HMM cross‑task generalization through parametric learning, which is not observed in other long‑context enhanced Mamba variants. Evaluations on Passkey Retrieval and LongBench-E tasks demonstrate that HMM improves retrieval success by 34.3-37.1\% and reasoning accuracy by 1.6-14.2\% over strong Mamba‑based models, while adding only 2\% extra parameters and with minimal training overhead.
\end{abstract}

\section{Introduction}
Transformer-based architectures have been the dominant choice of Large Language Models (LLMs) due to their strong expressive power \cite{vaswani2017attention, katharopoulos2020transformers}. However, their quadratic computational complexity $O(L^2)$ with respect to sequence length $L$ poses severe scalability challenges for long-context modeling. Recurrent linear attention models (RLAs), including RWKV \cite{peng2023rwkv}, linear transformers with the delta rule \cite{schlag2021linear, yang2024parallelizing}, and modern State-Space Models (SSMs) \cite{gu2022efficiently, zhong2026dyn} such as Mamba \cite{gu2024mamba}, address this limitation by performing recurrent updates that scale linearly with sequence length, offering a promising alternative for processing virtually unlimited sequences. Despite their theoretical potential for infinite context-length, empirical studies reveal that RLAs suffer from performance degradation when inference length exceeds the training horizon \cite{orvieto2023resurrecting}. Recent efforts \cite{ye2025longmamba,ben-kish2025decimamba,azizi2025mambaextend} focus on mitigating cumulative decay within the recurrence to enlarge the effective receptive field. However, recent discoveries indicate that finite-capacity recurrent states inevitably suffer from representational collapse as the information allocated per token vanishes~\cite{wen2024rnnstransformersyetkey,chen2024stuffed,wang2025understandingmitigatingbottlenecksstate}. 

In this work, we first reveal empirically that reducing numerical decay can substantially improve perplexity (PPL), yet this improvement does not reliably translate to gains in long-context reasoning (Fig.~\ref{fig:right_aligned}). We further prove theoretically that this limitation is irreconcilable, stemming from a non-injective mapping that projects an exponentially expanding semantic history space onto a bounded state manifold, causing distinct historical trajectories to become indistinguishable—a phenomenon we refer to as \textit{semantic aliasing}. We argue that mitigating recurrent decay alone is insufficient, as fixed-dimensional states inherently act as lossy compressors of expanding historical contexts. Long-context failures of RLAs arise not only from cumulative decay, but also from the absence of explicit mechanisms to organize, preserve, and selectively reuse high-level semantic information. 


\begin{figure}[t]
    \centering
    \captionsetup{font=normalsize,labelfont=normalfont,textfont=normalfont}
    \includegraphics[width=\columnwidth]{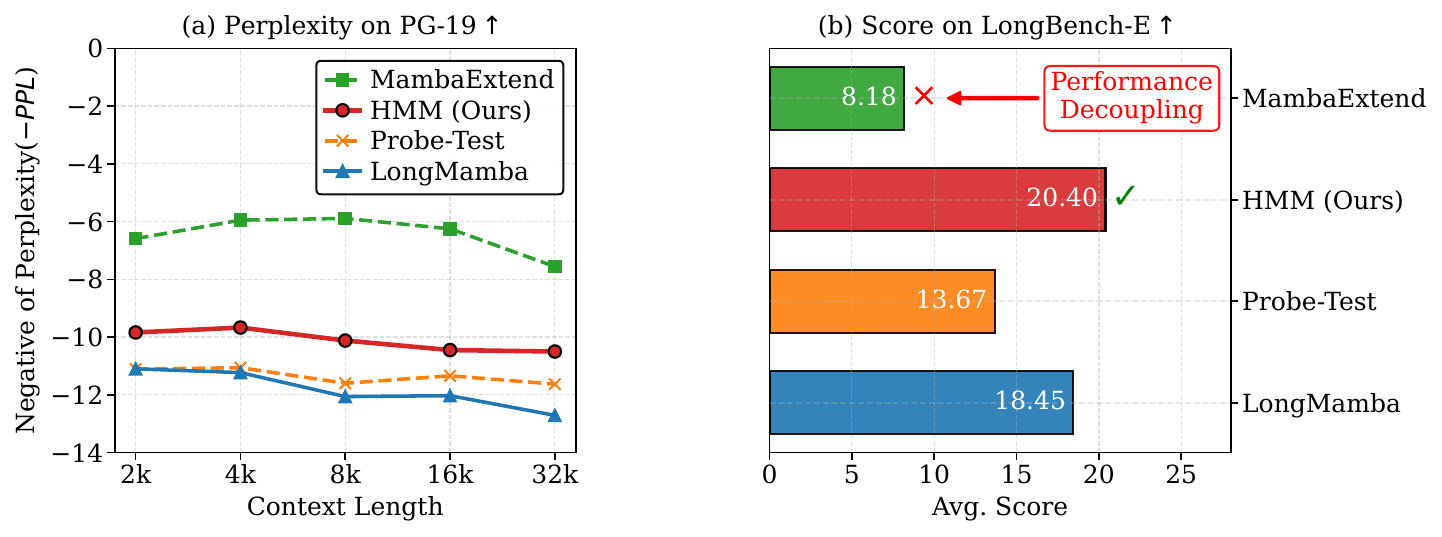}
    \vspace{-5pt}
    \caption{Good PPL on PG-19 in (a) does not translate to high reasoning accuracy on LongBench-E (b). HMM resolves this mismatch via hierarchical memory, mitigating semantic aliasing. Note that the y-axis in (a) denotes the negative of PPL for visual consistency with (b), higher is better.}
    \label{fig:right_aligned}
\end{figure}

Inspired by the multi-stage model of human memory \cite{sweller2003evolution,tulving2000oxford}, wherein information progresses from sensory to working memory and is consolidated into long-term storage for later retrieval, we propose Hierarchical Memory Mamba (HMM) (Fig.~\ref{fig:passkey_retrieval1}).
Building upon a pre‑trained Mamba backbone, HMM integrates a lightweight working memory network that extracts slow-evolving segments of paragraph‑level semantics (PLS) from the fast sensory memory embedded in the backbone’s hidden states. The PLS is subsequently encoded into highly-compressed array as persistent long‑term memory (LTM) for task‑relevant retrieval. Through hierarchical extraction, storage and retrieval of semantic information, HMM circumvents the fixed recurrent state limit of RLAs. Furthermore, the combination of parametric learning of the working memory and the non-parametric storage of LTM endows HMM cross‑task generalization, which is not observed in other long‑context enhanced Mamba variants. Evaluations on Passkey Retrieval and LongBench-E tasks demonstrate that HMM improves retrieval success by 34.3-37.1\% and reasoning accuracy by 1.6-14.2\% over strong Mamba‑based models, while adding only 2\% extra parameters with minimal training overhead.

\section{Related Work}

\textbf{Long Context Language Models.} Language models trained with limited context often degrade when evaluated on much longer inputs. While Transformer-based models extend context through specialized positional \cite{peng2023yarnefficientcontextwindow, ding2024longropeextendingllmcontext} or attention mechanisms \cite{zaheer2021bigbirdtransformerslonger, dao2022flashattentionfastmemoryefficientexact}, such techniques are not directly applicable to RLAs due to architectural difference. Recent work has started to close this gap. DeciMamba \cite{ben-kish2025decimamba} studies length extrapolation in Mamba and proposes training-free token filtering to reduce the effective sequence length in deeper layers. MambaExtend \cite{azizi2025mambaextend} instead performs training-free calibration of discretization-related scaling factors to improve robustness under length shift. LongMamba \cite{ye2025longmamba} analyzes channel-wise effective receptive fields and enlarges long-range capability via training-free token filtering targeted at global channels. However, these approaches largely address numerical decay under recurrent dynamics with fixed-capacity. Recent analyses suggest that the fundamental bottleneck of SSMs is representational rather than purely numerical: fixed-capacity states would lead to recency bias and representation over-smoothing as sequence length increases \cite{wen2024rnnstransformersyetkey,chen2024stuffed,wang2025understandingmitigatingbottlenecksstate}.

\textbf{Memory-Augmented Architectures.} 
Memory-augmented architectures improve long-context modeling through external space or structured memory mechanisms. Non-parametric methods like Retrieval-Augmented Generation (RAG)~\citep{lewis2020retrieval} and RETRO~\citep{borgeaud2022improving} utilize external databases to maintain expansive, up-to-date knowledge. Transformer-XL~\cite{dai2019transformerxlattentivelanguagemodels} and HMT~\cite{he2025hmthierarchicalmemorytransformer} introduce segment-level recurrent memory for Transformer architectures, however suffer from attention complexity under ultra-long contexts. 
RLAs such as Gated DeltaNet~\cite{yang2025gateddeltanetworksimproving} and MoM~\cite{du2025momlinearsequencemodeling} improves recurrent memory dynamics through gated forgetting or sparse routing. Nevertheless they still rely on continuously rewritten recurrent memories, which may lead to representation collapse under ultra-long contexts while introducing additional memory and parameter overhead. HMM addresses these limitations by using an RLA backbone while introducing a hierarchical memory that overcomes the representation bottleneck of fixed recurrent states, with minimal storage overhead.


\begin{figure*}[t] 
    \centering
    \includegraphics[width=\textwidth]{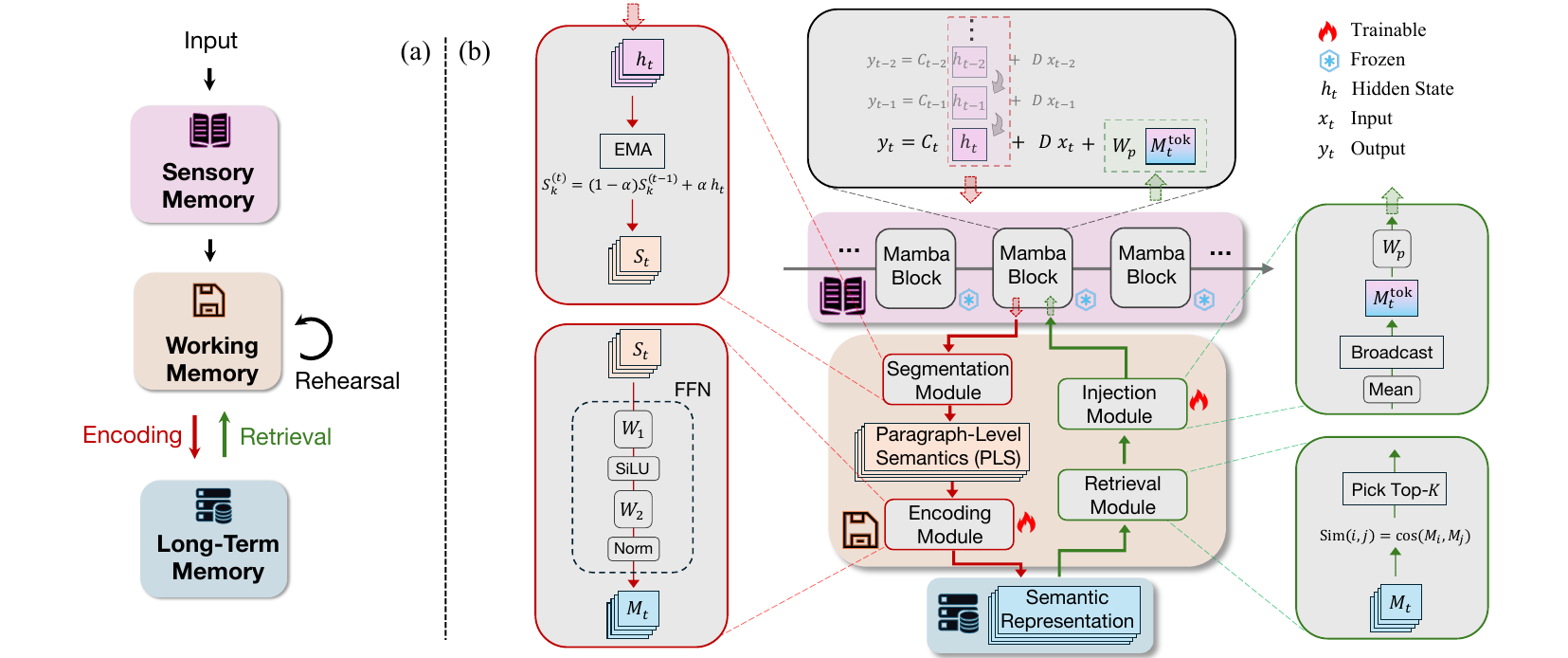}    
    \caption{Overview of HMM. (a) Cognitive Inspiration: Multi-stage human memory involving sensory, working, and long-term memory. (b) Technical Pipeline: Through a trainable working memory network, HMM firstly maps pre-trained Mamba hidden states ($h_t$) (resembling sensory memory) to high-level semantic representation as long-term memory via Semantic Encoding (red arrow, Sec.~\ref{sec:Semantic_encoding}), then retrieves and broadcasts them back to late layers of the pretrained backbone via Semantic Retrieval (green arrow, Sec.~\ref{sec:semantic_retrieval}) to enhance the capability for long-context reasoning.}
    \label{fig:passkey_retrieval1}
    \vspace{-0.5cm}
\end{figure*}

\section{Preliminary}
\label{sec:prelim}
We review the the recurrence design of Mamba~\cite{gu2024mamba} and analyze its long-context limitations in the lens of cumulative decay \cite{gu2022parameterization,wang2025understandingmitigatingbottlenecksstate}. 

\textbf{Mamba as a Time-Varying SSM.} Mamba is grounded in the S6 layer, which introduces input-dependent dynamics to the standard SSM. For a single channel, the recurrence is governed by
\begin{align}
    h_t &= \bar{\mathbf{A}}_t h_{t-1} + \bar{\mathbf{B}}_t x_t, \\
    \label{eq:s6_recurrence}
    y_t &= \mathbf{C}_t h_t,
\end{align}
where $x_t \in \mathbb{R}$ is the scalar input at time $t$, and $h_t \in \mathbb{R}^{d_h}$ denotes the latent hidden state with dimension $d_h$. The discretized system matrices $\bar{\mathbf{A}}_t \in \mathbb{R}^{d_h \times d_h}$ and $\bar{\mathbf{B}}_t \in \mathbb{R}^{d_h \times 1}$ are derived from their continuous-time counterparts via a step size $\Delta_t > 0$, which is predicted from the current input $x_t$ to enable selective focus. And $\mathbf{C}_t \in \mathbb{R}^{1 \times d_h}$ denotes the dicrete time-variant output matrices, and $y_t$ is the output.

For the convenience of theoretical analysis, we consider the following general linear time-varying form: 
\begin{equation}
    h_t = \Phi_t h_{t-1} + \Gamma_t x_t,
    \label{eq:ltv_ssm}
\end{equation}
where $\Phi_t \in \mathbb{R}^{d_h \times d_h}$ represents the effective state transition operator and $\Gamma_t \in \mathbb{R}^{d_h \times 1}$ is the input injection operator, both of which are input-dependent. 

\textbf{Cumulative Decay.} To evaluate the long-term retention of information, we unroll the recurrence to reveal the contribution of a past state $h_j$ to the current state $h_t$:
\begin{equation}
    h_t = \Bigl(\prod_{k=j+1}^{t} \Phi_k\Bigl) h_j + \sum_{k=j+1}^{t} \Bigl( \prod_{m=k+1}^{t} \Phi_m \Bigl) \Gamma_k x_k,
    \label{eq:unrolled_ssm}
\end{equation}
While Mamba's selection mechanism allows $\Phi_t \approx \mathbf{I}$ (where $\mathbf{I}$ is the identity matrix) to selectively preserve information, practical structural stability constraints require the spectral radius $\rho(\Phi_t) \le 1 - \epsilon$ on average to prevent numerical explosion~\cite{gu2022parameterization}, with $\epsilon > 0$ being a small stability margin. Consequently, the cumulative product decays exponentially as $\|\prod_{k=j+1}^{t} \Phi_k\| \sim O((1-\epsilon)^{t-j})$. This fading memory property inherently limits the model's ability to retain precise details over ultra-long horizons, as the gradient and signal from distant tokens $x_j$ diminish before they can be effectively utilized for reasoning at time $t$.
\section{Motivation}
\label{motivation}
\subsection{The Illusion of Low Perplexity}
\label{sec:illusion_ppl}
In long-context modeling, it is commonly assumed that numerical stability—typically measured by lower PPL—is sufficient to ensure the effectiveness of historical information \cite{ben-kish2025decimamba, ye2025longmamba}. To examine this assumption, we construct a diagnostic variant Probe that explicitly suppresses cumulative decay by generating a modulation factor for $\Delta t$ from the model’s own hidden states. Table~\ref{tab:ppl_analysis} compares the performance of this diagnostic variant Probe with baseline Mamba2~\citep{dao2024transformersssmsgeneralizedmodels} and LongMamba across different context lengths.
\begin{table}[htbp]
\centering
\caption{Comparison of PPL on Lambada and reasoning accuracy on LongBench (LB).}
\label{tab:ppl_analysis}
\small
\setlength{\tabcolsep}{3pt}  
\renewcommand{\arraystretch}{1.15}
\scalebox{0.9}{
\begin{tabular}{l|ccccc|c|c}
\hline
\textbf{Model} & \textbf{$2k$} & \textbf{$4k$} & \textbf{$8k$} & \textbf{$16k$} & \textbf{$32k$} & \textbf{Avg} & \textbf{LB} \\
\hline
Mamba2     & 13.18 & 265.6 & 3844  & 11196 & 17710 & 6606  & 7.1 \\
LongMamba  & 13.18 & 14.68 & 11.21 & 12.71 & 16.84 & 13.72 & \textbf{13.73} \\
Probe      & \textbf{13.09} & \textbf{14.27} & \textbf{10.34} & \textbf{11.79} & \textbf{15.73} & \textbf{13.00} & 9.57 \\
\hline
\end{tabular}}
\end{table}
The results show that LongMamba substantially stabilizes PPL relative to the baseline by alleviating cumulative decay. The Probe further reduces average PPL, outperforming LongMamba across all evaluated context lengths. However, this improvement in numerical stability does not translate into comparable gains on LongBench: the probe attains only 9.57 accuracy, markedly below LongMamba, revealing a decoupling between statistical fluency and semantic utility.

\subsection{Semantic Aliasing}
\label{sec:theoretical_analysis}

The decoupling between PPL and reasoning utility indicates that mitigating numerical decay alone is insufficient for Mamba in long-context modeling. We attribute this to a structural bottleneck: a fixed-dimensional recurrent state must compress an ever-growing semantic history, leading to \textit{semantic aliasing}. Recall the selective SSM recurrence in Eq.~\eqref{eq:ltv_ssm}. Let $\mathcal{F}_T:\mathcal{X}^T\to\mathcal{H}\subset\mathbb{R}^{d_h}$ denote the induced mapping from the length-$T$ history $x_{1:T}$ to the terminal state $h_T$.
Geometrically, $\mathcal{F}_T$ maps an exponentially expanding history space of all possible token sequences $\mathcal{X}^T$ into a bounded, fixed-dimensional manifold $\mathcal{H}$; hence it becomes non-injective at any finite resolution $\epsilon$ as $T$ grows~\cite{jelassi2024repeatmetransformersbetter,wen2024rnnstransformersyetkey}. 
We formally define \textit{semantic aliasing} as the phenomenon where two distinct histories $x_{1:T}^{(a)} \neq x_{1:T}^{(b)}$ result in states whose distance is below the model's discriminative threshold:
\begin{equation}
\big\|\mathcal{F}_T(x_{1:T}^{(a)})-\mathcal{F}_T(x_{1:T}^{(b)})\big\|_2 \le \epsilon .
\label{eq:semantic_aliasing}
\end{equation}
Appendix~\ref{sec:proof_aliasing} proves that such collisions are unavoidable beyond a context-length threshold (for fixed $d_h$ and $\epsilon$), implying that the model can become unable to resolve conflicting long-range semantics even without numerical decay. This motivates an external hierarchical memory design to circumvent the structural bottleneck, as illustrated in the next section.

\section{Method}
\label{method}
Inspired by hierarchical human memory, HMM adopts a pre-trained Mamba model as sensory memory backbone for token-level representation. A working memory is proposed to extract PLS from hidden states and further encode them into LTM for task-relevant retrieval. The operation of working memory consists of two core procedures: semantic encoding and semantic retrieval.

\subsection{Semantic Encoding}
\label{sec:Semantic_encoding}
Semantic encoding consists of a segmentation process that groups token representations into semantically coherent PLS, and a compression process that downsamples PLS into compact LTM preparing for subsequent retrieval.

\textbf{Segmentation.} Given hidden states ${h_t}$, we maintain a semantic prototype $S_k$ by measuring semantic consistency via cosine similarity. If $\cos(S_k, h_t) \ge \tau$, with $\tau$ denoting a threshold, the token is considered semantic consistent and updates the prototype via exponential moving average (EMA):
\begin{equation}
\label{ema}
S_k \leftarrow \mathrm{Norm}\big( (1-\alpha) S_k + \alpha h_t \big),
\end{equation}
where $\alpha$ is the EMA update rate. If $\cos(S_k, h_t) < \tau$, the token indicates semantic drift, and a counter $c_k$ accumulates consecutive low-similarity signals. A new segment prototype is created once $c_k > p$, controlled by the patience parameter $p$. This yields the PLS, $\{S_k\} \in \mathbb{R}^{K \times D}$, with $K$ and $D$ denoting its total length and state dimension of the sensory-memory backbone, respectively. In this way, PLS summarizes the slow-evolving paragraph semantics, ensuring a bounded context and non-vanishing sensitivity to tokens. A suedo code of the segmentation algorithm is provided in Appendix \ref{ema_segmentation}.

\textbf{Compression.} Each element of PLS is then downsampled via a lightweight feed-forward network into a compressed memory space, defined as LTM:
\begin{equation}
\label{eq:ffn_memory}
M_k = \mathrm{LN}\!\left( \mathbf{W}_2 \, \sigma\!\left( \mathbf{W}_1 S_k \right) \right),
\end{equation}
where $\mathbf{W}_1 \in \mathbb{R}^{S \times D}$ and $\mathbf{W}_2 \in \mathbb{R}^{M \times S}$ are learnable weights, $S$ is the dimensionality of an intermediate projection space,  and $M$ is the dimension of the state space of LTM. $\sigma(\cdot)$ denotes a non-linear activation (SiLU) and $\mathrm{LN}(\cdot)$ denotes Layer Normalization. Note that the historical PLS need not to be kept in memory, and its latest element can be discarded once it is compressed into LTM. By ensuring $M<<D$, the compression module achieves minimal storage overhead of LTM, $\{M_k\} \in \mathbb{R}^{K \times M}$.

\subsection{Semantic Retrieval}
\label{sec:semantic_retrieval}
After encoding, the working memory selects task relevant LTM via similarity ranking, then injects them back to the sensory backbone, adding useful semantic and context information for the task.

\textbf{Selection.}
The correlation between the latest element of LTM M and other historical elements is calculated by cosine similarity:
\begin{equation}
\mathrm{Sim}(i,j)
= \cos(M_i, M_j)
= \frac{{M}_i^\top {M}_j}{\lVert {M}_i \rVert \, \lVert {M}_j \rVert},
\qquad j < i .
\end{equation}
Then, the top-$K$ relevant historical LTM elements are selected via similarity ranking:
\begin{equation}
\label{eq:topk}
\mathcal{R}_{i} = \mathrm{topK}_{j} \big( \mathrm{Sim}_{i,j} \big),
\end{equation}
where $\mathcal{R}_{i}$ denotes the indices of the most semantically similar elements in LTM to $M_i$. The selection process identifies the most relevant historical LTM element indices for the current representation, providing input for subsequent injection.

\textbf{Injection.} 
The selected LTM elements are averaged to form a stable representation of relevant memory for the current task: 
\begin{equation}
\label{eq:mem_agg}
{M}^{\text{sta}}_i
= \frac{1}{|\mathcal{R}_i|} \sum_{r \in \mathcal{R}_i} {M}_r.
\end{equation}
Then, ${M}^{\text{sta}}_i$ is broadcast to all token-level representations within the paragraph summarized by $M_i$, forming a token-wise tensor ${M}_t^{\text{tok}}$, 
which is further upsampled to token-level dimension for injection onto the backbone:
\begin{equation}
\label{injection}
\tilde{\mathbf{y}}_t = \mathbf{y}_t + \mathbf{W}_p\,{M}^{\text{tok}}_t ,
\end{equation}
where $\mathbf{W}_p \in \mathbb{R}^{D \times M}$ is the upsampling weights. In this way, relevant historical memory is retrieved and integrated into the sensory backbone, circumventing the representation bottleneck of fixed recurrent states. 

\subsection{Theoretical Analysis} 
\label{sec:theore semantic}
The hierarchical memory of HMM mitigates semantic aliasing by preventing Jacobian rank collapse in standard recurrent dynamics. In standard SSMs, due to stability constraints, the Jacobian $\frac{\partial y_T}{\partial x_t}$, which measures the sensitivity of the current output to a distant input, involves a product of transition operators whose singular values decay exponentially, leading to a collapse in the sensitivity to distant inputs.

HMM addresses this by decomposing the total Jacobian into a recurrent path and a hierarchical memory path:
\begin{equation}
J_{HMM} = \frac{\partial y_{T}}{\partial x_{t}} + \frac{\partial (W_{p} {M}^{\text{tok}}_T)}{\partial S_{k^*}} \frac{\partial S_{k^*}}{\partial x_{t}} .
\end{equation}
The first term corresponds to the recurrent sensitivity path, which propagates through the product of state transition operators and vanishes as $T-t\to\infty$ under stability constraints. The second term is induced by hierarchical memory path, which injects relevant historical memory directly into the output. Since the PLS $S_{k^*}$-with $k^*$ denoting the indices of PLS elements selected by Top-$K$ retrieval-are constructed via local temporal aggregation (Eq.~\ref{ema}), the sensitivity $\frac{\partial S_{k^*}}{\partial x_t}$ depends only on the paragraph length and is independent of the total sequence length $T$. As a result, the hierarchical memory preserves a non-vanishing sensitivity path even when the recurrent Jacobian collapses. A formal analyse is in Appendix~\ref{sec:proof_jacobian} and Appendix~\ref{app:other_jac}.
\section{Experiments}
In this section, we conduct a comprehensive evaluation of HMM across a suite of complementary benchmarks to assess its long context modeling capabilities. 

\subsection{Experimental Settings}

\textbf{Datasets.} We evaluate HMM on PG-19 \cite{rae2019PG19}, LongBench-E \cite{bai2024longbench} and Passkey Retrieval \cite{azizi2025mambaextend}. 
We measure PPL on PG-19 following the setup in~\cite{ye2025longmamba} to assess fundamental modeling capabilities. LongBench-E is a comprehensive dataset comprising diverse real-world tasks such as long-form question answering and summarization. The Passkey Retrieval task is adopted to probe the effective context length and the model's ability to retrieve specific information from extended sequences.

\textbf{Training Setting.} We adopt pretrained Mamba2~\citep{dao2024transformersssmsgeneralizedmodels} models at different scales (780M and 1.3B) as the backbone for HMM. Unless otherwise specified, the pretrained Mamba backbone is kept frozen and only the parametric working memory is optimized. The additional trainable parameters introduced by \textsc{HMM} account for only $\sim$2\% of the backbone. All adaptation is performed \emph{exclusively on sequences sampled from the Pile dataset}~\citep{gao2020pile}, while evaluation is conducted on other downstream datasets \emph{without additional fine-tuning}. Note that this setting is distinct from other long-context Mamba variants where fine-tuning were done on downstream datasets, exhibiting cross-task generalization of HMM. Although \textsc{HMM} already shows clear improvements with only 300 adaptation samples, we report results adapted on 6{,}000 sequences to present the best achievable performance.
\begin{wraptable}[13]{r}{0.52\textwidth}
    \vspace{-0.8em}
    \centering
    \caption{PPL over different evaluation context lengths on PG-19 dataset. The Lower is better.}
    \label{tab:mamba_ppl}
    \footnotesize
    \setlength{\tabcolsep}{2.1pt}
    \renewcommand{\arraystretch}{0.92}
    \begin{tabular}{lcccc}
        \toprule
        \textbf{Method} & \textbf{$4k$} & \textbf{$8k$} & \textbf{$16k$} & \textbf{$32k$} \\
        \midrule
        \multicolumn{5}{l}{\textit{\textbf{Mamba2-780M}}} \\ 
        Vanilla   & 21.26 & 2320.31 & 7507.38 & 7531.29 \\
        LongMamba & 11.23 & 12.06 & 12.03 & 12.71 \\
        HMM       & \textbf{11.03} & \textbf{11.35} & \textbf{11.11} & \textbf{11.59} \\
        \midrule
        \multicolumn{5}{l}{\textit{\textbf{Mamba2-1.3B}}} \\
        Vanilla   &  14.81 & 129.33 & 1142.41 & 4642.43 \\
        LongMamba &  10.30 & 11.44 & 11.72 & 12.79 \\
        HMM       &  \textbf{10.02} & \textbf{10.47} & \textbf{10.98} & \textbf{12.16} \\
        \bottomrule
    \end{tabular}
    \vspace{-0.8em}
\end{wraptable}
\textbf{Implementation Details.} We employ a two-stage hyperparameter selection strategy. First, segmentation parameters ($\tau, p$) are calibrated on independent long-form samples to ensure generalizable boundary detection. Second, architectural configurations—including semantic encoding layer indices (defaulting to 12) and LTM state dimension (defaulting to 128)—are optimized on a LongBench-E validation split. Univariate ablations on the robustness of these choices are detailed in Sec.~\ref{sec:Ablation Studies of HMM}. Note that the working memory only encodes semantics from a single layer of the backbone and the retrieved memory is projected to all subsequent layers after the encoding layer.

\subsection{Long Context Modeling}
\textbf{Language Modeling.} Following the setup in LongMamba~\cite{ye2025longmamba}, we evaluate PPL on PG-19~\cite{rae2019PG19} using Mamba2-780M and Mamba2-1.3B backbones on sequences up to $32k$. As shown in Table~\ref{tab:mamba_ppl}, HMM consistently achieves the lowest PPL across all lengths, 
significantly outperforming vanilla Mamba and surpassing LongMamba, demonstrating superior long-context language modeling capability.

\noindent\textbf{Reasoning on LongBench-E.} 
We evaluate \textsc{HMM} on LongBench-E~\cite{bai2024longbench} following standard protocols~\cite{ye2025longmamba,ben-kish2025decimamba,azizi2025mambaextend,ye2025lamb}. 
As shown in Table~\ref{tab:longbench_results}, HMM achieves the best average accuracy among Mamba-based models, with \textbf{20.40\%} on Mamba2-780M and \textbf{22.50\%} on Mamba2-1.3B. 
The gains are most evident on TR, TQA, MN, LCC, and RB, where successful prediction requires stable access to task-relevant semantics across distant segments rather than local token-level retention. 
This suggests that HMM benefits from its hierarchical semantic memory, which provides a less decaying semantic access path in long sequences. 
In contrast, methods such as MambaExtend and LongMamba mainly improve state dynamics or receptive fields, but do not explicitly organize historical context into semantic units, making HMM more effective on tasks requiring global semantic coherence.

\begin{table*}[t]
    \centering
    \caption{Benchmark results on LongBench-E tasks. We compare \textsc{HMM} against Vanilla Mamba and recent long-context Mamba variants across various task categories. For task name abbreviations, refer to Appendix~\ref{Task Name}.}
    \scalebox{1.0}{
    \resizebox{\textwidth}{!}{
    \begin{tabular}{cl cccccccccccccc}
        \toprule
        \multirow{2}{*}{\textbf{Model}} & \multirow{2}{*}{\textbf{Method}} & 
        \multicolumn{2}{c}{\textbf{Synthetic}} & 
        \multicolumn{2}{c}{\textbf{Summary}} & 
        \multicolumn{2}{c}{\textbf{Single-doc QA}} & 
        \multicolumn{2}{c}{\textbf{Multi-doc QA}} & 
        \multicolumn{3}{c}{\textbf{Few-shot Learning}} & 
        \multicolumn{2}{c}{\textbf{Coding}} & 
        \multirow{2}{*}{\textbf{Avg.}} \\
        
        \cmidrule(lr){3-4} \cmidrule(lr){5-6} \cmidrule(lr){7-8} \cmidrule(lr){9-10} \cmidrule(lr){11-13} \cmidrule(lr){14-15}
        
         & & PC & PR & GR & MN & MQA & QA & 2WM & HQA & SS & TR & TQA & LCC & RB & \\
        \midrule
        \multirow{1}{*}{Mamba-2.8B}
            & DeciMamba & 0.50 & 1.50 & 14.86 & 24.58 & 18.58 & 8.91 & 9.06 & 4.46 & 7.34 & 0.50 & 12.61 & 8.67 & 10.96 & 9.43 \\
        \midrule
        
        \multirow{6}{*}{Mamba2-780M} 
            & Vanilla     & 0.78 & 2.76 & 5.84 & 7.40 & 3.45 & 1.53 & 1.41 & 1.13 & 2.23 & 9.00 & 6.35 & 25.83 & 12.64 & 6.18 \\
            & MambaExtend & 0.00 & 0.00 & 0.00 & 0.02 & 2.17 & \textbf{7.96} & \textbf{11.15} & \textbf{8.29} & 0.89 & 21.00 & 11.00 & 26.28 & 17.55 & 8.18\\
            & LongMamba   & \underline{0.90} & \textbf{3.79} & \underline{11.50} & \underline{17.22} & 10.11 & 4.13 & 6.28 & 4.72 & 14.7 & 32.5 & 43.13 & \textbf{47.77} & \underline{43.05} & 18.45 \\
            & LAMB      &\textbf{2.01} & \underline{3.44} & \textbf{12.19} & 9.42 & \underline{11.22} & 3.56 & 5.99 & 5.22 & \textbf{25.83} & \underline{38.46} & \underline{46.47} & 44.24 & 36.2 & \underline{18.79} \\
            & HMM         & 0.75 & 3.18 & 11.4 & \textbf{17.34} & \textbf{12.08} & \underline{4.49} & \underline{8.25}  & \underline{5.66} & \underline{18.59} & \textbf{44.5} & \textbf{49.31} & \underline{46.00} & \textbf{43.69} & \textbf{20.40} \\
        \midrule
        
        \multirow{6}{*}{Mamba2-1.3B} 
            & Vanilla     & 1.29 & 0.81 & 7.66 & 11.46 & 5.84 & 2.19 & 2.31 & 2.88 & 4.69 & 14.67 & 10.08 & 25.71 & 22.85 & 8.65 \\
            
            & MambaExtend & 0.00 & 0.00 & 0.01 & 0.02 & 0.70 & \textbf{7.34} & \underline{6.33} & \textbf{9.50} & 0.00 & 15.00 & 12.00 & 25.19 & 17.64 & 7.21 \\
            & LongMamba   & \underline{2.02} & \underline{3.51} & \textbf{14.33} & 10.28 & \textbf{14.73} & 5.14 & 5.73 & 5.52 & 14.00 & 21.67 & \underline{48.74} & \underline{42.99} & \underline{36.75} & 17.34 \\
            &LAMB       & \textbf{2.90} & \textbf{4.11} & \underline{13.09} & \underline{11.97} & \underline{12.63} & 5.72 & 5.76 & \underline{7.10} & \textbf{22.83} & \underline{51.28} & 48.03 & 38.85 & 32.34 & \underline{19.74} \\
            & HMM         & 1.22 & 1.25 & 12.54 & \textbf{21.3} & \underline{12.63} & \underline{6.83} & \textbf{9.89}  & 5.63 & \underline{19.12} & \textbf{52.5} & \textbf{51.41} & \textbf{49.83} & \textbf{48.32} & \textbf{22.50}  \\
            
        \bottomrule
    \end{tabular}%
    }}
    \label{tab:longbench_results}
\end{table*}

\noindent\textbf{Passkey Retrieval.} Beyond long-context perplexity, exact-retrieval tasks provide a more stringent test of whether models can preserve and recover specific information from distant positions~\cite{liu-etal-2024-lost}. We employ the Passkey Retrieval task following the MambaExtend protocol~\cite{azizi2025mambaextend}. Both vanilla Mamba and HMM undergo lightweight fine-tuning (20 steps) solely for format alignment. We evaluate exact retrieval accuracy across sequence lengths ranging from $1k$ to $64k$ at varying depths, HMM improves retrieval success by 34.3-37.1\%, with results presented in Fig.~\ref{fig:passkey_retrieval}. 
This result indicates that hierarchical semantic memory offers an effective alternative to discretization-step calibration for long-range recall.

\begin{figure}[t]
    \centering
    \captionsetup{font=normalsize,labelfont=normalfont,textfont=normalfont}
    \includegraphics[width=0.98\textwidth]{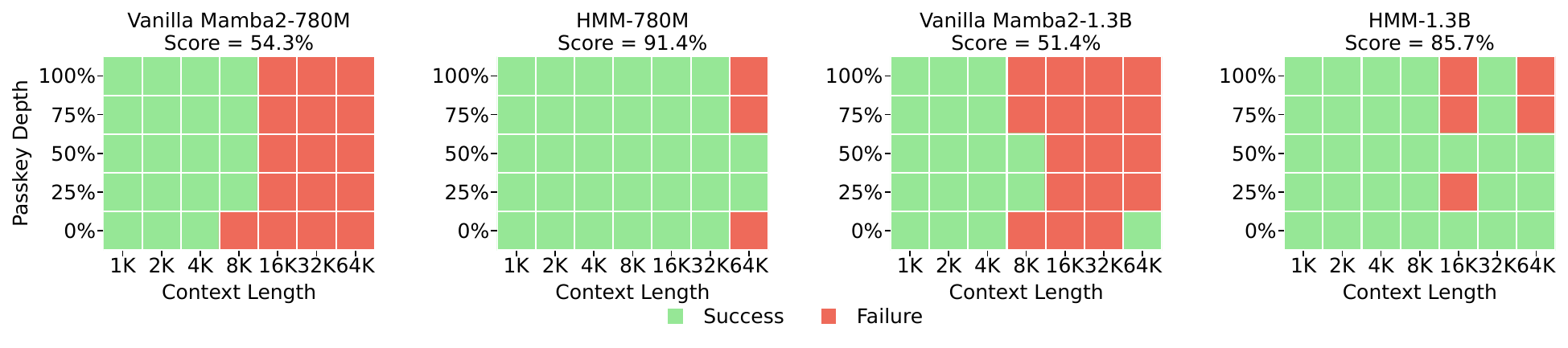}
    \vspace{-0.25cm}
    \caption{Passkey retrieval accuracy heatmaps. Green indicates success and red indicates failure.}
    \vspace{-0.45cm}
    \label{fig:passkey_retrieval}
\end{figure}

\noindent\textbf{Memory Overhead and Throughput Analysis.} As shown in Fig.\ref{fig:memory_tau_layer}a, we measure the peak GPU memory usage under different context lengths and compare \textsc{HMM} with the Mamba2 baseline under identical batch size and numerical precision settings. Across all sequence lengths, HMM matches the peak memory footprint of Mamba2, incurring no additional memory overhead and preserving the efficiency and scalability of SSMs.




\vspace{-0.2cm}
\begin{figure}[H]
    \centering
    \captionsetup{font=normalsize,labelfont=normalfont,textfont=normalfont}

    \includegraphics[width=\linewidth,
        trim=0cm 7cm 0cm 3cm,
        clip]{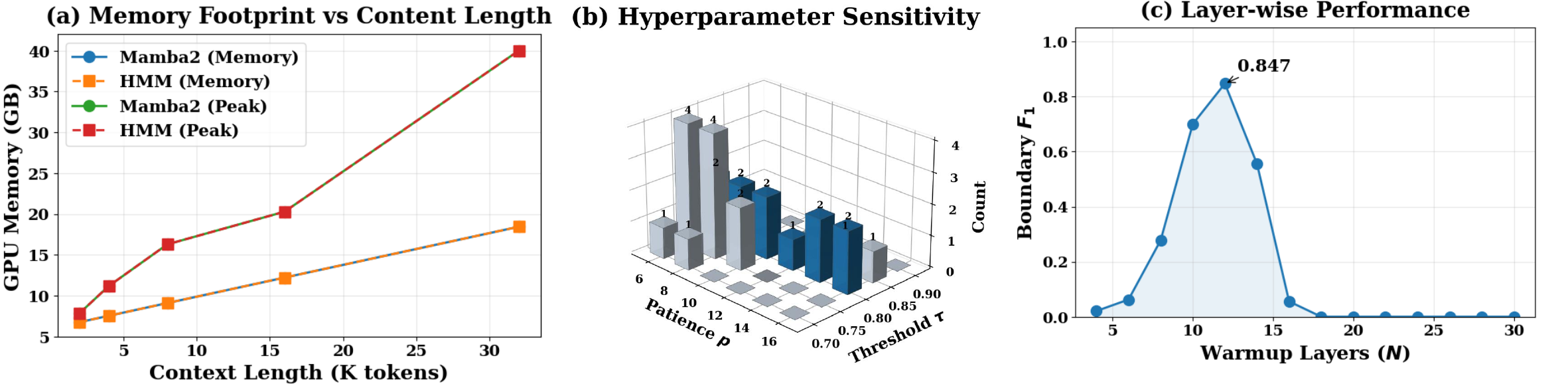}
    \vspace{-0.15cm}
    \caption{
    (a) Memory footprint comparison between Mamba2 and HMM across context lengths.
    (b) Sensitivity to segmentation hyperparameters $(\tau, p)$.
    (c) Layer-wise boundary detection performance.
    }
    \vspace{-0.45cm}
    \label{fig:memory_tau_layer}
\end{figure}
Also, we report detailed throughput results across varying context lengths (e.g., 2k--32k). As shown in Table~\ref{tab:latency}, HMM introduces only a modest prefill overhead ($\approx$1.13$\times$--1.24$\times$). During decoding, HMM achieves throughput comparable to Mamba, as paragraph-level memory is computed once during prefill and reused without additional per-step overhead.
\vspace{-0.4cm}
\begin{table}[H]
\caption{Latency comparison}
\centering
\small
\setlength{\tabcolsep}{5pt}
\begin{tabular}{l c c c c c}
\toprule
Model & Context Length & Prefill (s) & Prefill (K tok/s) & Decode (s) & Decode (tok/s) \\
\midrule
Mamba2 & 2048  & 0.058 & 35.40 & 4.21 & 30.40 \\
HMM    & 2048  & 0.072 & 28.32 & 4.40 & 29.07 \\
\midrule
Mamba2 & 4096  & 0.074 & 55.26 & 4.53 & 28.29 \\
HMM    & 4096  & 0.084 & 48.50 & 4.63 & 27.63 \\
\midrule
Mamba2 & 8192  & 0.126 & 64.79 & 4.36 & 29.34 \\
HMM    & 8192  & 0.151 & 54.68 & 4.54 & 28.19 \\
\midrule
Mamba2 & 16384 & 0.249 & 65.67 & 4.10 & 31.22 \\
HMM    & 16384 & 0.292 & 56.04 & 4.31 & 29.71 \\
\bottomrule
\end{tabular}

\label{tab:latency}
\end{table}
\vspace{-0.5cm}
\subsection{Ablation Study}
\label{sec:Ablation Studies of HMM}
We perform ablation studies from four perspectives:
(1) Hyperparameter analysis: the influence of threshold and patience parameters $(\tau, p)$ and the choice of semantic encoding layer;
(2) Effectiveness of hierarchical memory: comparison with a parameter-matched Mamba baseline without hierarchical memory;
(3) The optimal memory compression rate;
(4) Comparison to Transformer baseline and RAG method.
Unless otherwise specified, all experiments are conducted on 130M model scale.

\noindent\textbf{Segmentation Hyperparameters.} 
To analyse the effect of $(\tau, p)$ on segmentation performance, we perform ablations on five randomly sampled multi-topic long-form sequences from the Pile, fixing the injection layer to isolate structural factors. We sweep $\tau \in \{0.7, 0.75, 0.8, 0.85, 0.9\}$ and $p \in \{6, 8, 10, 12, 14, 16\}$.
As shown in Fig.~\ref{fig:memory_tau_layer}b, the segmentation quality and downstream performance vary smoothly across a broad range of values, indicating that the proposed segmentation mechanism is not sensitive to precise hyperparameter tuning. 


\noindent\textbf{Semantic Encoding Layer.} To identify the optimal layer for semantic encoding, we evaluate the semantic segmentation quality with different PLS extraction layers from the backbone with controlled samples. Fig.~\ref{fig:memory_tau_layer}(c) presents the segmentation results on long sequences sampled from the Pile dataset. A higher F1 score indicates better alignment with the underlying semantic boundaries. For the 1.3B model, the segmentation accuracy reaches its peak around layer 12 (48 layers in total). This aligns well with the intuition that sensory memory primarily resides in the early network layers. Detailed experimental settings are provided in Appendix~\ref{app:segmentation_eval}. 

\textbf{Effectiveness of Hierarchical Memory.}
To rule out long-context fine-tuning as the source of performance gains, we conduct a controlled comparison on sequences lengths ranging up to $32k$. Both \textsc{HMM} and the Mamba backbone are trained on the same 300 long-form sample sequences; \textsc{HMM} updates only its working memory, while the baseline unfreezes an equal number of backbone parameters. As shown in Tables~\ref{tab:ppl_finetune_300} and~\ref{tab:acc_finetune_300}, the parameter-matched Mamba baseline fails to achieve comparable improvements, confirming that the gains stem from the hierarchy of HMM rather than long-context fine-tuning.

\vspace{-0.3cm}
\begin{table}[H]
\centering
\caption{PPL comparison under different context lengths.}
\label{tab:ppl_finetune_300}
\scalebox{0.95}{
\begin{tabular}{lcccc}
\toprule
\textbf{Model} & \textbf{$4k$} & \textbf{$8k$} & \textbf{$16k$} & \textbf{$32k$} \\
\midrule
Mamba2  & 21.19 & 776.89 & 2364.04 & 2856.38 \\
HMM & \textbf{16.25} & \textbf{17.01}  & \textbf{19.87}   & \textbf{29.91}   \\
\bottomrule
\end{tabular}}
\vspace{-0.35cm}
\end{table}
\vspace{-0.3cm}
\begin{table}[H]
\vspace{-0.15cm}
\centering
\caption{Acc. (\%) comparison on long-context reasoning tasks.}
\label{tab:acc_finetune_300}
\scalebox{0.95}{
\begin{tabular}{lcccccc}
\toprule
\textbf{Model} & \textbf{2WM} & \textbf{QA} & \textbf{MN} & \textbf{HQA} & \textbf{GR} & \textbf{Avg.} \\
\midrule
Mamba2 & 2.53 & 1.74 & 5.80 & 0.46 & 3.44 & 2.80 \\
HMM & \textbf{8.98} & \textbf{4.59} & \textbf{15.11} & \textbf{5.57} & \textbf{10.42} & \textbf{8.90} \\
\bottomrule
\end{tabular}}
\vspace{-0.4cm}
\end{table}

\textbf{Memory Compression.} We study the optimal memory compression rate by varying the LTM state dimension and check the performance of HMM. As shown in Table~\ref{tab:dmem_ablation}, increasing LTM dimension improves performance up to a moderate scale, after which the gain saturates. We therefore set $M=128$, providing a balanced trade-off between expressivity and efficiency.
\vspace{-0.3cm}
\begin{table}[H]
\centering
\small
\caption{Effect of LTM dimension on performance.}
\label{tab:dmem_ablation}

\setlength{\tabcolsep}{4pt}

\begin{tabular}{c|c|cccccc}
\toprule
$d_{\text{mem}}$
& PPL$\downarrow$
& 2WM
& QA
& MN
& HQA
& GR
& Avg$\uparrow$ \\
\midrule

64
& 327.38
& 0.78
& 0.02
& 1.07
& 0.01
& 0.56
& 0.49 \\

96
& 9.83
& 8.05
& 3.06
& 13.18
& 4.33
& 5.38
& 6.80 \\

128
& \textbf{8.92}
& \textbf{8.56}
& \textbf{3.58}
& \textbf{13.58}
& \textbf{4.73}
& \textbf{5.86}
& \textbf{7.26} \\

160
& 10.89
& 6.79
& 3.17
& 12.77
& 4.05
& 5.72
& 6.50 \\

192
& 36.82
& 2.40
& 2.68
& 7.32
& 2.29
& 4.17
& 3.77 \\

256
& 315.28
& 1.96
& 0.16
& 2.33
& 0.21
& 0.89
& 1.11 \\

\bottomrule
\end{tabular}
\end{table}
\vspace{-0.2cm}
\textbf{Comparison to Transformer.} We include results from Pythia-1.4B, which is adopted as a Transformer baseline of the HMM's Mamba backbone. This enables a more meaningful comparison under similar training conditions. As shown in Table~\ref{tab:pythia_longbench}, HMM even slightly surpasses the baseline, indicating the potential of hierarchical memory.
\vspace{-0.2cm}
\begin{table}[H]
\centering
\small
\caption{Comparison with the Transformer baseline Pythia-1.4B on LongBench}
\setlength{\tabcolsep}{6pt}
\begin{tabular}{l c c c c c c}
\toprule
Model 
& Single-Doc 
& Multi-Doc 
& Summary 
& Synthetic 
& Code \& Class 
& Avg ($\uparrow$) \\
\midrule
Pythia-1.4B & 13.17 & 4.49 & 13.97 & 2.32 & 33.94 & 20.21 \\
HMM-1.3B    & \textbf{11.83} & \textbf{8.03} & \textbf{14.33} & \textbf{1.81} & \textbf{49.08} & \textbf{21.23} \\
\bottomrule
\end{tabular}
\label{tab:pythia_longbench}
\end{table}
\vspace{-0.2cm}
\textbf{Comparison to RAG.} Finally, we compare HMM with a standard RAG-enhanced Mamba\cite{lewis2020retrieval} baseline on both PPL and LongBench-E. As shown in Fig.~\ref{fig:RAG_Comparison}, HMM achieves consistently better perplexity and long-context reasoning performance. We attribute this to HMM’s PLS compression and retrieval within the recurrent modeling process, which provides more stable long-range semantic access. In contrast, RAG-based methods rely on external chunks concatenation, which introduces fragmented context representations and additional prefill overhead. HMM achieves ($\approx$1.6$\times$--2.3$\times$) higher prefill throughput than RAG-Mamba under long-context settings. Detailed throughput and latency analyses are provided in Appendix~\ref{RAG-latency}.




\begin{figure}[H]
    \centering
    \includegraphics[
        width=\linewidth,
        trim=2cm 4cm 2cm 4cm,
        clip
    ]{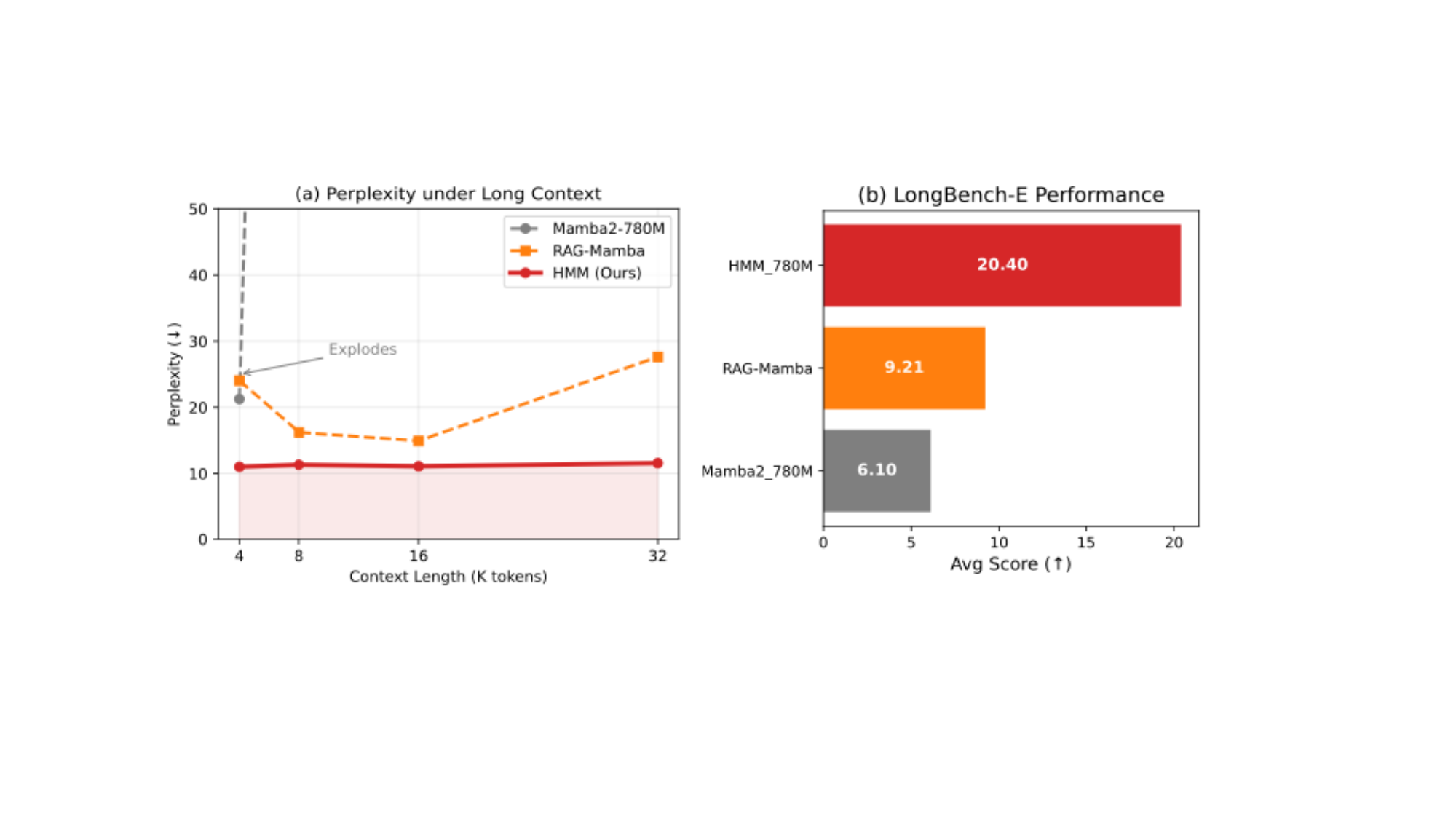}
    \vspace{-1cm}
    \caption{
    (a) Perplexity comparison on different context lengths.
    (b) Comparison on Longbench-E.
    }
    
    \label{fig:RAG_Comparison}
\end{figure}

\section{Conclusion}
In this work, we uncover the inherent decoupling between PPL and long-context reasoning in previous works, demonstrating that optimizing PPL alone is insufficient for effective long-context modeling. We further present theoretical analyses and identify the semantic aliasing issue prevalent in RLAs. Inspired by the hierarchical human memory, we propose the HMM which circumvents the representation bottleneck imposed by fixed recurrent states. By integrating parametric working memory for semantic encoding and retrieval alongside non-parametric LTM storage, HMM achieves strong cross-task generalization and consistently outperforms existing long-context-enhanced Mamba variants. Our work reveals the potential of brain-inspired hierarchical memory design in mitigating inherent representation limitations of recurrent architectures and advancing long-context modeling capabilities. More discussions of the limitations and broader impacts of this work are presented in Appendix~\ref{LI} and ~\ref{BI}.

\clearpage
\bibliographystyle{plain}
\bibliography{example_paper}

\clearpage
\appendix

\section{Semantic Aliasing}
\label{sec:proof_aliasing}
Semantic aliasing means that distinct histories map to nearly indistinguishable states, so the model cannot resolve conflicting contexts. We show that semantic aliasing is unavoidable for a fixed-dimensional recurrent state.
Although Mamba uses input-dependent selection in the recurrence $h_t=\Phi_t h_{t-1}+\Gamma_t x_t$ (Eq.~\eqref{eq:ltv_ssm}), it does not change the state dimension $d_h$, which limits the number of states that remain distinguishable at a finite resolution.

\begin{theorem}[Inevitability of $\epsilon$-aliasing]
\label{the:aliasing}
Let $\mathcal{H} \subset \mathbb{R}^{d_{h}}$ be the set of reachable hidden states, assumed to be bounded such that $\mathcal{H} \subseteq R\mathcal{B}_{2}^{d_{h}}$ (a ball of radius $R$). Let $V$ be the vocabulary. If the sequence length $T$ satisfies:
\begin{equation}
T > \frac{d_{h} \log_{2}(3R/\epsilon)}{\log_{2}|V|}, 
\end{equation}
then the mapping $\mathcal{F}_{T}: V^{T} \rightarrow \mathcal{H}$ is not injective at resolution $\epsilon$. That is, there exist at least two distinct histories $x_{1:T}^{(a)} \neq x_{1:T}^{(b)}$ such that $\|\mathcal{F}_{T}(x_{1:T}^{(a)}) - \mathcal{F}_{T}(x_{1:T}^{(b)})\|_{2} \leq \epsilon$.
\end{theorem}

\begin{proof}
Let $\mathcal{H}\subset\mathbb{R}^{d_h}$ denote the set of reachable hidden states under the data distribution.
We assume $\mathcal{H}$ is bounded: there exists $R>0$ such that
$\mathcal{H}\subseteq R B_2^{d_h}$, where $B_2^{d_h}=\{h\in\mathbb{R}^{d_h}:\|h\|_2\le 1\}$.
This is a standard stability assumption for discretized SSM dynamics with bounded inputs as commonly enforced in practice for Mamba-style models~\cite{gu2024mamba}.

Let $\mathcal{N}(\mathcal{H},\epsilon)$ be the $\epsilon$-covering number of $\mathcal{H}$ under the $\ell_2$ metric,
i.e., the minimum number of $\ell_2$ balls of radius $\epsilon$ needed to cover $\mathcal{H}$. We use $\ell_2$ to quantify indistinguishability in the hidden space; other norms lead to analogous statements with the corresponding covering bounds. Intuitively, $\mathcal{N}(\mathcal{H},\epsilon)$ upper bounds how many state ``slots'' remain reliably distinguishable at resolution $\epsilon$. Since $\mathcal{H}\subseteq R B_2^{d_h}$, a standard volumetric bound gives
\begin{equation}
\mathcal{N}(\mathcal{H},\epsilon)
\le \mathcal{N}(R B_2^{d_h},\epsilon)
\le \left(1+\frac{2R}{\epsilon}\right)^{d_h}
\le \left(\frac{3R}{\epsilon}\right)^{d_h}.
\label{eq:covering_number}
\end{equation}

Let $\mathcal{V}$ be the vocabulary and $\mathcal{X}^T=\mathcal{V}^T$ the set of length-$T$ token histories.
Let $\mathcal{F}_T:\mathcal{X}^T\to\mathcal{H}$ be the history-to-state mapping induced by Eq.~\eqref{eq:ltv_ssm},
defined by $\mathcal{F}_T(x_{1:T})=h_T$.

By Eq.~\eqref{eq:covering_number}, the number of $\epsilon$-distinguishable regions in $\mathcal{H}$ is at most
$\mathcal{N}(\mathcal{H},\epsilon)\le (3R/\epsilon)^{d_h}$.
Meanwhile, the number of distinct length-$T$ token histories is $|\mathcal{X}^T|=|\mathcal{V}|^T$.
If $|\mathcal{V}|^T>\mathcal{N}(\mathcal{H},\epsilon)$, then by the pigeonhole principle,
two distinct histories must fall into the same $\epsilon$-ball, implying
$\|\mathcal{F}_T(x^{(a)}_{1:T})-\mathcal{F}_T(x^{(b)}_{1:T})\|_2\le\epsilon$.
Rearranging $|\mathcal{V}|^T>(3R/\epsilon)^{d_h}$ yields the stated condition on $T$.
\end{proof}

In Mamba, the input $x_t\in\mathbb{R}^{d_x}$ is a continuous token embedding.
However, standard language modeling draws tokens from a finite vocabulary $\mathcal{V}$, and the embedding table induces a finite set
$\mathcal{E}=\{e(v):v\in\mathcal{V}\}\subset\mathbb{R}^{d_x}$.
Therefore, the realizable embedding-history space $\mathcal{X}^T_{\mathrm{emb}}$ contains the discrete subset $\mathcal{E}^T$ with
$|\mathcal{E}^T|=|\mathcal{V}|^T$.
Thus, counting $|\mathcal{V}|^T$ provides a valid \emph{lower bound} on the number of distinguishable histories; allowing more general continuous inputs can only enlarge the set of possible histories and cannot remove $\epsilon$-aliasing.

If distinct histories collapse to indistinguishable states at $\epsilon$-resolution, then any state-based readout cannot reliably condition predictions on their long-range differences, yielding an intrinsic accuracy ceiling on retrieval-style long-context tasks. This perspective aligns with recent analyses of recency and over-smoothing in SSMs~\cite{wang2025understandingmitigatingbottlenecksstate} and retrieval bottlenecks in fixed-size recurrent models~\cite{wen2024rnnstransformersyetkey}.

\section{Details of EMA-based Semantic Segmentation}
\label{ema_segmentation}

\begin{algorithm}[H]
\SetAlgoLined
\DontPrintSemicolon
\caption{EMA-based semantic segmentation and PLS construction}
\label{alg:prototype_ema}

\KwIn{Hidden states $\mathbf{h}_{1:L}\in\mathbb{R}^{L\times D}$, threshold $\tau$, patience $p$, EMA rate $\alpha$, minimum segment length $m$}
\KwOut{PLS sequence, $\mathbf{S}=(S_1,S_2,\ldots,S_K)$}

Normalize hidden states: 

$\mathbf{h}_t \leftarrow \mathrm{Norm}(\mathbf{h}_t)$ for $t=1,\dots,L$\;

Initialize current prototype $S_k \leftarrow \mathbf{h}_1$\;

Initialize prototype list $\mathbf{S} \leftarrow [\ ]$\;

Initialize drift counter $c_k \leftarrow 0$, 

segment length $\ell_k \leftarrow 1$\;

\For{$t = 2$ \KwTo $L$}{
    Compute semantic similarity:\;
    $\mathrm{sim} \leftarrow \cos(S_k, \mathbf{h}_t)$\;

    \If{$\mathrm{sim} \ge \tau$}{
        $S_k \leftarrow \mathrm{Norm}\big((1-\alpha) S_k + \alpha\, \mathbf{h}_t\big)$\;
        
        $c_k \leftarrow 0$\;
        
        $\ell_k \leftarrow \ell_k + 1$\;
    }\Else{
        $c_k \leftarrow c_k + 1$\;
        
        \If{$c_k > p$ \textbf{and} $\ell_k \ge m$}{
            Append $S_k$ to $\mathbf{S}$\;
            
            Initialize new prototype: 
            $S_k \leftarrow \mathbf{h}_t$\;
            
            $c_k \leftarrow 0$, 
            
            $\ell_k \leftarrow 1$\;
        }
    }
}

Append final $S_k$ to $\mathbf{S}$\;

\Return $\mathbf{S}$\;
\end{algorithm}

\section{Mitigating Semantic Aliasing via HMM}
\label{sec:proof_jacobian}

\begin{theorem}[Jacobian Rank Preservation]
\label{the:jaco}
Let $\tilde{y}_T$ be the output of HMM at time $T$, and $x_t$ be the input at time $t$ ($t < T$). In the presence of the hierarchical retrieval mechanism, the Jacobian $\mathbf{J}_{\mathrm{HMM}}$ maintains a non-vanishing rank component; that is, the rank of $\mathbf{J}_{\mathrm{HMM}}$ does not decay to zero as $(T-t)\to\infty$.
\end{theorem}

\begin{proof}

In a standard Mamba or SSM, the hidden state dynamics follow
\begin{equation}
h_i = \Phi_i h_{i-1} + \Gamma_i x_i ,
\end{equation}
where $\Phi_i \in \mathbb{R}^{d_h \times d_h}$ and $\Gamma_i \in \mathbb{R}^{d_h \times d_{\mathrm{in}}}$ are the discretized transition and input matrices. The output is given by $y_T = \mathbf{C} h_T$. By the chain rule, the sensitivity of the output to a distant input $x_t$ through the recurrent backbone is
\begin{equation}
\label{eq:J_recurrent_def}
\frac{\partial y_T}{\partial x_t} =\mathbf{C}\left( \prod_{i=t+1}^{T} \Phi_i \right)\Gamma_t .
\end{equation}

For numerical stability, the spectral radius satisfies $\rho(\Phi_i) < 1$. Consequently, the singular values of the transition product $\prod_{i=t+1}^{T} \Phi_i$ decay exponentially as $(T-t)\to\infty$, leading to a structural collapse of the recurrent Jacobian:$\lim_{(T-t)\to\infty} \left\|\frac{\partial y_T}{\partial x_t}\right\|= 0$.

Unlike standard SSMs, HMM injects retrieved semantic information via an external hierarchical memory. Following the semantic retrieval steps in Sec.~\ref{method}, the augmented state is
\begin{equation}
\label{eq:h_hmm}
h_T^{\mathrm{HMM}}=(\Phi_T h_{T-1} + \Gamma_T x_T) + \mathbf{W}_p M^{\text{tok}}_T,
\qquad M^{\text{tok}}_T = \mathrm{TopK}(\mathcal{R}_T; q_T),
\end{equation}
where $\mathcal{R}_T=\{M_k\}$ is the LTM, $q_T$ is the working memory query,
and $\mathbf{W}_p$ is the injection projection defined in Sec.~\ref{sec:semantic_retrieval}.

Accordingly, the total Jacobian in HMM admits the decomposition
\begin{equation}
\label{eq:J_HMM_decomp}
\mathbf{J}_{\mathrm{HMM}}=\underbrace{\frac{\partial y_T}{\partial x_t}}_{\mathbf{J}_{\mathrm{recurrent}}} +\underbrace{\mathbf{C} \mathbf{W}_p\frac{\partial M^{\text{tok}}_T}{\partial S_{k^*}}\frac{\partial S_{k^*}}{\partial x_t}}_{\mathbf{J}_{\mathrm{retrieval}}} ,
\end{equation}
where $k^*$ denotes the semantic prototype selected by causal Top-$K$ retrieval.

As shown above, $\mathbf{J}_{\mathrm{recurrent}}$ vanishes exponentially as $(T-t)\to\infty$. In contrast, the retrieval term $\mathbf{J}_{\mathrm{retrieval}}$ is governed by the sensitivity of the selected prototype $S_{k^*}$
to the original input $x_t$. Since $S_{k^*}$ is constructed via local temporal aggregation over a segment of fixed length $L_{\mathrm{seg}}$, this sensitivity is independent of the total sequence length $T$:
\begin{equation}
\left\|
\frac{\partial S_{k^*}}{\partial x_t}
\right\|
\ge \gamma > 0 ,
\end{equation}
for some constant $\gamma$ determined by the local encoding function (Sec.~\ref{sec:Semantic_encoding}).

Using the rank inequality $\mathrm{rank}(A+B)\ge \mathrm{rank}(B)-\mathrm{rank}(A)$ and noting that $\mathrm{rank}(\mathbf{J}_{\mathrm{recurrent}})\to 0$, we obtain
\begin{equation}
\lim_{(T-t)\to\infty}
\mathrm{rank}(\mathbf{J}_{\mathrm{HMM}})\;\ge\;\mathrm{rank} \left(\mathbf{C} \mathbf{W}_p\frac{\partial M^{\text{tok}}_T}{\partial S_{k^*}}\frac{\partial S_{k^*}}{\partial x_t}\right) .
\end{equation}

Therefore, the hierarchical memory provides a non-recursive shortcut that preserves a constant-rank Jacobian component.
This ensures non-vanishing sensitivity to distant inputs and completes the proof.
\end{proof}

\section{Theoretical Analysis of the Architecture of HMM}
\label{app:other_jac}

In this section, we provide a formal comparison between HMM
and common architectural alternatives for mitigating long-context failures in SSMs.
To address representational compression and sensitivity decay, one may consider two intuitive strategies:
(i) Dimension Scaling, which expands the hidden state dimension $d_h$ to increase capacity;
and (ii) Internal Path Optimization, which incorporates internal shortcuts or gating within the recurrence.
We show that neither approach can simultaneously guarantee Jacobian rank recovery and computational efficiency,
whereas HMM achieves both by decoupling semantic memory from the recurrent dynamics.

\subsection{Computational Efficiency: Quadratic vs. Linear Scaling}

A straightforward approach to increasing memory capacity is to enlarge the hidden dimension $d_h$ of the SSM.
However, this strategy yields limited gains in effective representational rank while incurring prohibitive computational cost.

For a sequence of length $T$, mitigating the representational compression loss (Theorem~\ref{the:aliasing})
requires $d_h$ to grow with the entropy of the history.
Yet the computational complexity of the SSM recurrence
(e.g., parallel scan or recurrent update) scales quadratically with the state dimension:
\begin{equation}
\mathcal{C}_{\mathrm{SSM}} \approx O(T \cdot d_h^2).
\end{equation}
As $T$ increases, maintaining bounded distortion necessitates $d_h$ to scale at least logarithmically,
and often linearly, with $T$, resulting in super-linear or cubic complexity.
This renders dimension scaling impractical for ultra-long sequences.

In contrast, HMM decouples memory capacity from the recurrent state. Historical information is stored as paragraph-level semantic prototypes $\{S_k\}$, constructed via coarse-grained temporal aggregation (Sec.~\ref{sec:Semantic_encoding}), while retrieval selects only a small set of relevant prototypes via Top-$K$ similarity (Sec.~\ref{sec:semantic_retrieval}). The overall complexity is therefore
\begin{equation}
\mathcal{C}_{\mathrm{HMM}} \approx
\underbrace{O(T \cdot d_h^2)}_{\text{Recurrent Transition}}
+ \underbrace{O(T \cdot K \cdot d_h)}_{\text{Top-$K$ Semantic Retrieval}},
\end{equation}
where $K$ is a small constant. Since retrieval operates on paragraph-level representations and only $K$ prototypes are activated at each step, HMM preserves linear scaling in $T$ without increasing the core recurrence width.

\subsection{Jacobian Rank Recovery: Internal vs. External Paths}
\label{sec:appendix_jacobian_comparison}

Another possible mitigation is to introduce internal shortcuts (e.g., gating or residual connections) within the recurrent dynamics. However, such internal modifications remain subject to the same stability constraints
that govern standard SSMs.

For any linear recurrent system
\begin{equation}
h_t = \Phi_t h_{t-1} + \Gamma_t x_t ,
\end{equation}
the Jacobian of the hidden state with respect to a past input $x_j$ propagates via the chain rule:
\begin{equation}
\mathbf{J}^{\mathrm{SSM}}_{t,j}
= \frac{\partial h_t}{\partial x_j}= \left( \prod_{k=j+1}^{t} \Phi_k \right) \Gamma_j .
\end{equation}
To ensure numerical stability, the spectral radius must satisfy $\rho(\Phi_k) \le 1-\epsilon$ for some $\epsilon>0$.
Consequently, the spectral norm of $\mathbf{J}^{\mathrm{SSM}}_{t,j}$ decays exponentially as $(t-j)\to\infty$,
leading to an unavoidable collapse in sensitivity and Jacobian rank.
This phenomenon underlies the semantic aliasing effect formalized in Theorem~\ref{the:aliasing}.

In contrast, HMM employs an external additive bypass: retrieved semantic inforamtion are injected into the state after the recurrent transition (Eq.~\ref{eq:h_hmm}). This bypass introduces a Jacobian component that does not propagate through the product of transition matrices. As shown in Appendix~\ref{sec:proof_jacobian}, the resulting Jacobian contains a constant-rank term arising from the retrieval pathway, independent of the recurrent decay. Thus, HMM restores long-range sensitivity without modifying the internal recurrence or increasing $d_h$.

\section{Training Recipes}

Unless otherwise specified, all models are trained using the AdamW optimizer with the following settings:
\begin{itemize}
    \item All adaptation experiments are conducted on a fixed set of 6{,}000 long-context sequences sampled from the Pile, with lengths ranging from 2k to 16k.
    \item Gradient clipping with a maximum norm of 1.0;
    \item Weight decay of 0.1 and No dropout;
    \item Linear learning rate warmup followed by cosine decay.
\end{itemize}

\begin{table}[h]
\centering
\small 
\caption{Model sizes and training hyperparameters used in scaling experiments.}
\label{tab:scaling_core}
\begin{tabular}{lccccc}
\toprule
\textbf{Params} & \textbf{n\_layers} & \textbf{d\_model} & \textbf{Training Steps} & \textbf{Learning Rate} & \textbf{Tokens} \\
\midrule
780M & 24 & 1536 & 6{,}000 & 2.5e{-}4 & 50M  \\
1.3B & 24 & 2048 & 6{,}000 & 2e{-}4   & 50M  \\
\bottomrule
\end{tabular}
\end{table}

\section{Semantic Segmentation Evaluation Protocol}
\label{app:segmentation_eval}

To evaluate the optimal semantic encoding layers on the backbone, we construct long sequences by concatenating multiple documents with different topics sampled from the Pile dataset. Since document boundaries naturally correspond to topic transitions, these concatenation points are treated as ground-truth semantic boundaries.

For each evaluated layer, we extract token-level hidden states and apply the proposed EMA-based segmentation algorithm to generate predicted semantic boundaries. Specifically, segmentation is triggered when the cosine similarity between the current token representation and the running segment prototype falls below a threshold $\tau$ for more than $p$ consecutive tokens, subject to a minimum segment length constraint.

We evaluate segmentation quality using boundary-level F1 score. A predicted boundary is considered correct if it falls within a small tolerance window around a ground-truth boundary. Higher F1 scores indicate better alignment between predicted segments and underlying semantic transitions, reflecting stronger semantic resolution in the corresponding hidden representations.

For the experiments in Fig.~\ref{fig:memory_tau_layer}(c), we evaluate multiple layers of the 1.3B Mamba2 model using long sequences sampled from the Pile dataset. The segmentation hyperparameters are selected based on validation performance, with $\tau \in [0.8]$ and patience parameter $p \in [10]$.

\section{RAG Latency Comparison}
\label{RAG-latency}
In this section, we compare the prefill latency and throughput of \textsc{HMM}-780M, Vanilla Mamba2-780M, and a RAG-enhanced Mamba2-780M baseline across different context lengths. For the RAG\cite{lewis2020retrieval} baseline, the input context is first divided into chunks, after which task-relevant chunks are retrieved and concatenated back into the input sequence for generation. Results are reported in Table~\ref{tab:rag_latency}. We observe that while RAG-Mamba introduces significant overhead due to retrieval, \textsc{HMM} maintains efficient prefill performance with only modest overhead compared to Mamba2.

\begin{table}[H]
\centering
\small
\setlength{\tabcolsep}{6pt}
\caption{Latency comparison across models. Prefill throughput is measured in K tokens/s.}
\label{tab:rag_latency}
\begin{tabular}{lcccc}
\toprule
Model & Context & Prefill (s) & Prefill (K tok/s) \\
\midrule
Mamba2-780M & 2048 & 0.058 & 35.40 \\
RAG-Mamba   & 2048 & 0.092 & 24.00 \\
HMM         & 2048 & 0.072 & 28.32 \\
\midrule
Mamba2-780M & 4096 & 0.074 & 55.26 \\
RAG-Mamba   & 4096 & 0.187 & 30.00 \\
HMM         & 4096 & 0.084 & 48.50 \\
\midrule
Mamba2-780M & 8192 & 0.126 & 64.79 \\
RAG-Mamba   & 8192 & 0.184 & 24.00 \\
HMM         & 8192 & 0.151 & 54.68 \\
\bottomrule
\end{tabular}
\end{table}

\section{Task Name Abbreviations}
\label{Task Name}
\noindent \textbf{LongBench-E Tasks:}
\begin{itemize}
    \item \textbf{2WM}: 2WikiMQA
    \item \textbf{GR}: GovReport
    \item \textbf{HQA}: HotpotQA
    \item \textbf{LCC}: LCC
    \item \textbf{MQA}: MultiFieldQA-en
    \item \textbf{MN}: MultiNews
    \item \textbf{PC}: Passage Count
    \item \textbf{PR}: PassageRetrieval-en
    \item \textbf{QA}: Qasper
    \item \textbf{RB}: RepoBench-P
    \item \textbf{SS}: SAMSum
    \item \textbf{TR}: TREC
    \item \textbf{TQA}: TriviaQA
\end{itemize}

\vspace{0.5cm} 

\noindent \textbf{LongBench Tasks:}
\begin{itemize}
    \item \textbf{PC}: Passage Count
    \item \textbf{SS}: SAMSum
    \item \textbf{2WM}: 2WikiMQA
    \item \textbf{TQA}: TriviaQA
    \item \textbf{QA}: Qasper
    \item \textbf{VCS}: VCSUM (zh)
    \item \textbf{MSQ}: Musique
    \item \textbf{MN}: MultiNews
    \item \textbf{QMS}: QMSum
    \item \textbf{HQA}: HotpotQA
    \item \textbf{LCC}: LCC
    \item \textbf{NQA}: NarrativeQA
    \item \textbf{DR}: DuReader (zh)
    \item \textbf{MQAzh}: MultiFieldQA-zh
    \item \textbf{MQA}: MultiFieldQA-en
    \item \textbf{GR}: GovReport
    \item \textbf{LS}: LSHT (zh)
    \item \textbf{RB}: RepoBench-P
    \item \textbf{PR}: PassageRetrieval-en
    \item \textbf{PRzh}: PassageRetrieval-zh
    \item \textbf{TR}: TREC
\end{itemize}

\section{Limitations}
\label{LI}
While HMM improves long-context modeling for Mamba-based architectures, it is not specifically optimized for tasks that require precise retrieval over large collections of factual information, such as legal or regulatory analysis. In such scenarios, where accurate access to extensive historical facts is critical, the current segmentation and retrieval mechanisms may be insufficient, and more advanced retrieval or indexing strategies are needed.

In ultra-long contexts, the primary bottleneck of linear models such as Mamba lies in the accessibility problem caused by recurrent state compression, where useful historical information becomes difficult to retrieve despite being theoretically retained. The PLS mitigates this issue by preserving stable semantic anchors over long ranges, thereby improving effective information accessibility without significantly degrading useful details in practice. However, this design focuses on improving accessibility rather than enabling fine-grained or exact retrieval over large factual corpora, which remains an open challenge.

Future work may focus on improving the segmentation strategy, enhancing the retrieval module, and extending the framework to better support tasks that demand precise and reliable information access.

\section{Statistical Significant Test}
We conducted significance tests comparing HMM against the strongest baseline model (i.e., LongMamba).
Utilizing two backbone models (Mamba2-780M and Mamba2-1.3B) and a random seed range of 1 to 3, we ran HMM and each of the baseline models a total of 3 × 2 = 6 times.

\section{Broader Impacts}
\label{BI}
This work advances long-context sequence modeling by introducing a hierarchical memory mechanism that overcomes the inherent fixed-capacity limitation of RLAs. By enabling structured extraction, storage, and retrieval of semantic memories, HMM improves the utilization of long-range information. Compared to existing Mamba-based variants, HMM achieves significant gains in both retrieval and reasoning performance while introducing only 2\% additional parameters and minimal training overhead.

Furthermore, this work suggests a promising direction for building more efficient and scalable sequence models. By leveraging architectural design rather than increasing model size, HMM improves performance while maintaining computational efficiency. Combined with the low memory footprint of linear models and performance comparable to Transformer-based models, HMM has strong potential for real-world deployment in resource-constrained settings, such as embodied AI systems and edge devices, where efficiency, memory usage, and real-time processing are critical.


\end{document}